\documentclass{article}

\usepackage{arxiv}

\usepackage{microtype}
\usepackage{graphicx}
\usepackage{subfigure}
\usepackage{booktabs}

\usepackage[square,numbers,sort]{natbib}

\usepackage{amsmath}
\usepackage{amssymb}
\usepackage{mathtools}
\usepackage{amsthm}
\usepackage{bm}
\usepackage{xurl}
\usepackage{enumitem}
\usepackage{tikz}
\usepackage{pgfplots}
\usetikzlibrary{patterns}
\pgfplotsset{compat=1.17}

\usepackage{hyperref}
\usepackage{url}
\usepackage{microtype}

\usepackage{colortbl}

\usepackage[capitalize,noabbrev]{cleveref}

\crefname{equation}{Eq.}{Eqs.}

\theoremstyle{plain}
\newtheorem{theorem}{Theorem}[section]
\newtheorem{proposition}[theorem]{Proposition}
\newtheorem{lemma}[theorem]{Lemma}

\theoremstyle{definition}

\theoremstyle{remark}

\usepackage[textsize=tiny]{todonotes}

\DeclarePairedDelimiter{\paren}{\lparen}{\rparen}

\DeclarePairedDelimiter{\sqbra}{\lbrack}{\rbrack}
\DeclarePairedDelimiter{\bra}{\lbrace}{\rbrace}

\DeclarePairedDelimiter{\norm}{\lVert}{\rVert}
\DeclarePairedDelimiterX{\pgiven}[2]{\lparen}{\rparen}{#1\,\delimsize\vert\,#2}
\DeclarePairedDelimiterX{\qgiven}[2]{\lbrack}{\rbrack}{#1\,\delimsize\vert\,#2}
\DeclarePairedDelimiterX{\inner}[2]{\lparen}{\rparen}{#1\cdot#2}

\newcommand{\diag}[1]{\mathrm{diag}\paren*{#1}}
\newcommand{\tr}[1]{\mathrm{tr}\paren*{#1}}
\newcommand{\rank}[1]{\mathrm{rank}\paren*{#1}}

\newcommand{\Relu}[1]{\mathrm{ReLU}\paren*{#1}}
\newcommand{\st}{\mathrm{s.t.}}
\newcommand{\argmax}{\mathop{\rm arg~max}\limits}

\newcommand{\LeakyRelu}[1]{\mathrm{LeakyReLU}\paren*{#1}}
\theoremstyle{plain}
\newtheorem{problem}[theorem]{Problem}
\usepackage{multirow}

\title{Interior-Point Vanishing Problem in Semidefinite Relaxations for Neural Network Verification}
\author{
    \begin{tabular}{cc}
        \begin{tabular}[t]{c} 
            Ryota Ueda \\
            \rm{Institute of Science Tokyo} \\
            \texttt{ueda.r.sss@gmail.com}
        \end{tabular}
        &
        \begin{tabular}[t]{c} 
            Takami Sato \\
            \rm{Keio University} \\
            \texttt{takami.sato@yos.elec.keio.ac.jp}
        \end{tabular}
        \\ \\ 
        \begin{tabular}[t]{c} 
            Ken Kobayashi \\
            \rm{Institute of Science Tokyo} \\
            \texttt{kobayashi.k@iee.eng.isct.ac.jp}
        \end{tabular}
        &
        \begin{tabular}[t]{c} 
            Kazuhide Nakata \\
            \rm{Institute of Science Tokyo} \\
            \texttt{nakata.k.ac@m.titech.ac.jp}
        \end{tabular}
    \end{tabular}
}

\date{}

\begin{document}
\maketitle

\begin{abstract}

Semidefinite programming (SDP) relaxation has emerged as a promising approach for neural network verification, offering tighter bounds than other convex relaxation methods for deep neural networks (DNNs) with ReLU activations. However, we identify a critical limitation in the SDP relaxation when applied to deep networks: interior-point vanishing, which leads to the loss of strict feasibility -- a crucial condition for the numerical stability and optimality of SDP. 
 Through rigorous theoretical and empirical analysis, we demonstrate that as the depth of DNNs increases, the strict feasibility is likely to be lost, creating a fundamental barrier to scaling SDP-based verification. To address the interior-point vanishing,  we design and investigate five solutions to enhance the feasibility conditions of the verification problem. Our methods can successfully solve 88\% of the problems that could not be solved by existing methods, accounting for 41\% of the total. Our analysis also reveals that the valid constraints for the lower and upper bounds for each ReLU unit are traditionally inherited from prior work without solid reasons, but are actually not only unbeneficial but also even harmful to the problem's feasibility.
This work provides valuable insights into the fundamental challenges of SDP-based DNN verification and offers practical solutions to improve its applicability to deeper neural networks, contributing to the development of more reliable and secure systems with DNNs.

\end{abstract}

\section{Introduction}
\label{sec:intro}

Deep Neural Networks (DNNs) have achieved remarkable success across various domains, including image recognition~\cite{krizhevsky2012imagenet}, natural language processing~\cite{bahdanau2014neural}, and autonomous systems~\cite{chen2024end}. 
However, their vulnerability against adversarial perturbations, known as adversarial attacks~\cite{Szegedy2014, goodfellow2014explaining}, has raised critical concerns about their reliability and security~\cite{li2023sok} and spurred extensive research efforts to enhance the robustness of DNN models such as robustified training (e.g., adversarial training~\cite{madry2017towards}) and certified defenses~\cite{wong2018provable,cohen2019certified}. 
Although these defense or mitigation methods have gained traction, they can so far only offer empirical or probabilistic guarantees and thus do not provide definitive guarantees about a model's behavior under adversarial conditions. To address the limitation, DNN verification~\cite{huang2017safety, cheng2017maximum, katz2017reluplex} has emerged as a critical area of research, aiming to ensure that DNN models behave reliably under all possible scenarios within assumed conditions, e.g., verifying whether a classification model maintains its predicted label for any input perturbation within a given norm ball.

Due to the high complexity of DNNs, complete (i.e., exact) verification methods~\cite{anderson2020strong,bastani2016measuring,botoeva2020efficient,ehlers2017formal,katz2022reluplex,lomuscio2017approach,tjeng2017evaluating} are often NP-hard and therefore too expensive in practice. Consequently, incomplete verification~\cite{henriksen2020efficient,singh2019abstract,weng2018towards,wang2021beta} has gained attention, which is a method to verify a relaxed version of the complete DNN verification problem since a DNN is always verified if its relaxed verification problem is verified.  Among the various approaches to relaxing the complete verification, semidefinite programming (SDP) relaxation~\cite{SDP-IP, batten2021efficient} has shown high performance for verifying DNNs with rectified linear unit (ReLU) activation units, as it is known to give the tightest relaxation among convex relaxation approaches~\cite{ehlers2017formal}. The current state-of-the-art incomplete approaches~\cite{xu2020fast, wang2021beta, zhou2024scalable} refine the solutions with branch-and-bound (BaB) strategies, e.g., exploring all combinations of positive and negative cases of each ReLU activation~\cite{bunel2018unified, wang2021beta}.  While BaB brings performance improvements in practice,  the quality of the initial solution obtained by the convex relaxation is still essential, as BaB is applied on top of it.

While SDP relaxation is recognized as a method that can give one of the tightest relaxations~\cite{chiu2023tight, batten2021efficient}, SDP-based verification methods are not commonly used in state-of-the-art verification methods~\cite{brix2024fifth}. Their high computational cost is often mentioned as a reason for this~\cite{li2023sok}. However, in this study, we identify a more fundamental reason that challenges the use of SDP in DNN verification: interior-point vanishing, a term we define to describe the phenomenon that SDP-based verification is unlikely to have feasible interior points (i.e., strict feasibility) when the depth of the DNN being verified increases. 
Strict feasibility is a crucial condition that ensures numerical stability and optimality in SDP solutions~\cite{sekiguchi2021perturbation}. The interior-point vanishing could be the actual reason hindering the development of the SDP-based approach in this area.

\noindent\textbf{Contributions:} This study is the first to identify the interior-point vanishing problem in SDP-based DNN verification and evaluate the impacts. Prior studies~\cite{batten2021efficient, lan2023semidefinite} did not identify this problem because they did not evaluate models as deep as practical DNNs to observe the impact.
Through our theoretical and empirical analysis, we reveal that as the depth of a network grows, the likelihood of encountering strict feasibility issues increases, posing a significant limitation to the applicability of SDP-based verification methods in practical scenarios.
To address the challenges, we propose five newly designed approaches to enhance the feasibility conditions of the verification problem. These methods aim to mitigate numerical instability and improve the solvability of the SDP. 
Our experimental results demonstrate that our methods can successfully solve 88\% of the problems that could not be solved by existing methods, accounting for 41\% of the total.
We further find that the valid constraints for the lower and upper bounds for each ReLU unit are traditionally inherited from prior work~\cite{batten2021efficient, SDP-IP} without solid reasons, but are actually not only unbeneficial but also harmful to the problem's feasibility.

\noindent\textbf{Notation:}
Let $\mathbb R^n$ and $\mathcal S^n$ denote the $n$-dimensional Euclidean space and the space of $n\times n$ symmetric matrices, respectively. 
We express a symmetric matrix $\bm X\in \mathcal S^n$ as $\bm X\succeq \bm O$ and $\bm X\succ \bm 0$ if $\bm X$ is positive semidefinite and positive definite, respectively.
For a vector $\bm x\in \mathbb R^n$, its $i$-th element is denoted by $\paren{\bm x}_i$. 
For a matrix $\bm X\in \mathbb R^{m\times n}$, the $\paren{i,j}$ element of $\bm X$ is denoted by $\paren{\bm X}_{ij}$, and the $i$-th row of $\bm X$ is denoted by $\bm X\paren{i,\colon}$. 
The $\ell_2$ and $\ell_\infty$ norms of a vector $\bm x$ are denoted by $\norm{\bm x}_2$ and $\norm{\bm x}_\infty$, respectively, and the Frobenius norm of a matrix $\bm X$ is denoted by $\norm{\bm X}_F = \sqrt{\tr{\bm X^\top \bm X}}$. 
For two vectors $\bm x, \bm y\in \mathbb R^n$, the Hadamard product, which is 
element-wise multiplication, is denoted by $\bm x\odot \bm y$. 
The map $\diag{\cdot} \colon \mathbb R^{n\times n}\to \mathbb R^n$ represents the operator that arranges the diagonal elements of a matrix into a vector.
For a nonnegative integer $L$, we define $\sqbra{L}\coloneqq \{0, 1, \dots, L-1\}$. 

\section{Background}
\label{sec:background}

\subsection{DNN Verification Problem}

Consider a DNN model $\bm f\colon \mathbb{R}^d \rightarrow \mathbb{R}^m$ for $m$-class classification. 
Let $\bar{\bm x}$ be an input instance and $i^\star$ be its prediction label provided by $\bm f$. 
We define the DNN verification problem as the following decision problem:

\begin{problem}[DNN Verification] \label{prob:verification}
    Given a DNN model $\bm f\colon \mathbb{R}^d \rightarrow \mathbb{R}^m$, an input instance $\bar{\bm x}$ with the prediction label $i^\star$, and a perturbation radius $\rho>0$, determine whether for all  $\bm x_0$ satisfying $\norm{\bm x_0 - \bar{\bm x}}_{\infty} \leq \rho$, the prediction label remains unchanged, i.e., $\argmax_{i=1,\dots,m}~\paren{\bm f\paren{\bm x_0}}_i =i^\star$.
\end{problem}

This problem determines whether the prediction label of the input $\bar{\bm x}$ can be changed to another label $i\neq i^\star$ by perturbing the input within the radius $\rho>0$. If \Cref{prob:verification} is true, the prediction label cannot change to another label $i\neq i^\star$ by perturbing the original input $\bar{\bm x}$ within the radius $\rho$, meaning the DNN model $\bm f$ is \emph{robust} for  $\bar{\bm x}$ with a perturbation radius $\rho$.  
In this work, we focus on $L$-layer feed-forward DNN $\bm f\colon \mathbb R^{d}\to \mathbb R^m$. For an input $\bar{\bm x}$, its output is defined as $\bm f\paren{\bar{\bm x}}\coloneqq \bm W_L\bm x_L + \bm b_L$, where $\bm x_{i+1} = \Relu{\bm W_i \bm x_{i} + \bm b_i}$ for $i \in \sqbra{L}$, and $\bm x_0 = \bar{\bm x}$, where $\bm W_i \in \mathbb R^{n_{i+1} \times n_i}$ and $\bm b_i \in \mathbb R^{n_{i+1}}$ are the weight matrix and bias vector of the $i$-th layer, respectively. 
$\Relu{\cdot}$ is a map that applies the ReLU function to each element of a vector. The predicted label for an input $\bar{\bm x}$ is determined as $i^\star = \argmax_{i=1,\dots, m}~\paren{\bm f\paren{\bar{\bm x}}}_i$.

For a DNN with ReLU activation, \Cref{prob:verification} can be equivalent to solve the following optimization problem:
\begin{subequations} \label{eq:main_problem}
\begin{alignat}{3}
\gamma^{\star}\coloneqq  \min_{\bra{\bm x_i}}&\quad \bm c^{\top} \bm x_L + c_0  \\
\st
&\quad \bm x_{i+1} = \operatorname{ReLU}\paren{\bm W_i \bm x_i + \bm b_i} &\quad \paren{i \in \sqbra{L}}, \label{eq:relu_constraint} \\
&\quad\norm{\bm x_0 - \bar{\bm x}}_{\infty} \leq \rho, \label{eq:perturbation} \\
&\quad\bm l_{i+1} \leq \bm x_{i+1} \leq \bm u_{i+1} &\quad \paren{ i \in \sqbra{L}},
\label{eq:layer} 
\end{alignat}
\end{subequations}
where $\bm c\coloneq \paren{\bm W_L\paren{i^\star, \colon}-\bm W_L\paren{i,\colon}}^\top $ and $c_0\coloneq \paren{\bm b_L}_{i^\star}-\paren{\bm b_L}_i$. 
The vectors $\bm l_{i+1}$ and $\bm u_{i+1}$ represent the lower and upper bounds after activation, respectively. We obtained the lower and upper bounds using a boundary propagation method \cite{wang2018efficient,henriksen2020efficient}. 
\Cref{prob:verification} is true if and only if the optimal value $\gamma^{\star}$ is positive. 
However, due to the constraint~\eqref{eq:relu_constraint}, this optimization problem is nonconvex and thus generally hard to solve in a practical time.

Instead of directly solving this complete verification problem, we can relax it to an incomplete convex relaxation problem by outer-approximating the feasible region of the problem~\eqref{eq:main_problem} with a convex set $\mathcal{D}$:
\begin{align*}
\gamma_{\mathcal{D}}\coloneqq \min_{\bra{\bm x_i}}&\quad \bm c^{\top} \bm x_L+c_0\\
    \st&\quad \paren{\bm x_0, \bm x_1, \ldots, \bm x_L} \in \mathcal{D}, 
\end{align*}
where $\gamma^{\star} \geq \gamma_{\mathcal{D}}$ holds. 
When $\gamma_{\mathcal{D}}>0$, the answer to \Cref{prob:verification} is true; Otherwise (i.e., when $\gamma^{\star}>0 \geq \gamma_{\mathcal{D}}$), the verification is undetermined.

\subsection{Semidefinite Relaxation} \label{sec:sdp_relax}

SDP relaxation~\cite{SDP-IP, batten2021efficient, chiu2023tight} is one of the outer-approximation methods for the nonconvex optimization problem~\eqref{eq:main_problem}, and recognized as a method that can give one of the tightest relaxations. Starting from the problem~\eqref{eq:main_problem}, we can obtain the SDP-relaxed problem as the following procedure: First, we convert the problem~\eqref{eq:main_problem} into an equivalent quadratically constrained quadratic program (QCQP). The ReLU constraints~\eqref{eq:relu_constraint} can be equivalently replaced with the following linear and quadratic constraints:
\begin{subequations}\label{eq:relu_qcqp}
\begin{alignat}{3}
 &\bm x_{i+1} \geq \bm 0 &\quad     \paren{i \in \sqbra{L}}, \\ 
 &\bm x_{i+1} \geq \bm W_i \bm x_i+\bm b_i&\quad  \paren{i \in \sqbra{L}}, \\
  &\bm x_{i+1} \odot \paren{\bm x_{i+1}-\bm W_i \bm x_i-\bm b_i} = 0 &\quad \paren{i \in \sqbra{L}}.   
\end{alignat}    
\end{subequations}

The input and activation constraints~\eqref{eq:perturbation} and \eqref{eq:layer} can be also replaced with the following quadratic ones:
\begin{equation} \label{eq:qcqp}
    \bm x_i \odot \bm x_i - \paren{\bm l_i + \bm u_i} \odot \bm x_i + \bm l_i \odot \bm u_i \leq 0 
    \quad \paren{i \in [L+1]}.
\end{equation}
We note that $i=0$ corresponds to the input layer. \Cref{eq:qcqp} for $i=0$ describes the perturbation bound on the input $\bar{\bm x}$, i.e., $\bm l_0 = \bar{\bm x} - \rho \bm {1}_{n_0}$ and $\bm u_0 = \bar{\bm x} + \rho \bm 1_{n_0}$ for the input layer, \Cref{eq:qcqp} with $L=0$ is equivalent to the input constraint~\eqref{eq:perturbation}. As \Cref{eq:relu_qcqp,eq:qcqp} are quadratic, the problem~\eqref{eq:main_problem} can be converted to the following QCQP:

\begin{equation}\label{prob:qcqp}
    \gamma^* \coloneqq \min\bra{\bm c^{\top} \bm x_L + \bm c_0 \mid \text{\Cref{eq:relu_qcqp,eq:qcqp}}}.    
\end{equation}

Finally, we can derive an SDP relaxation of the QCQP problem~\eqref{prob:qcqp} by polynomial lifting~\cite{parrilo2000structured,lasserre2009moments}.
We introduce a vector $\bm v = \paren{1, \bm x_0^{\top}, \bm x_1^{\top},\dots , \bm x_L^{\top}}^{\top} \in \mathbb{R}^{1 + \sum_{i=0}^L n_i}$ and a symmetric matrix variable $\bm P=\bm v \bm v^{\top} \in \mathcal S^{1 + \sum_{i=0}^L n_i}$. 
With $\bm v$ and $\bm P$, the problem~\eqref{prob:qcqp} can be equivalently reformulated as an SDP with a rank constraint: $\rank{\bm{P}} = 1$. 
By removing the rank constraint, we can obtain an SDP relaxation of the problem~\eqref{prob:qcqp} as follows:
 \begin{subequations} \label{eq:sdp_relax}
\begin{alignat}{3}
\min_{\bm P}&\quad  \bm{c}^{\top} \bm P\sqbra{\bm{x}_L} + c_0 \\
\st&\quad \bm P\sqbra{\bm{x}_{i+1}} \geq \bm{0} &\quad \paren{i \in\sqbra{L}},  \label{eq:sdp_ineq_nonneg}\\ 
&\quad \bm P\sqbra{\bm{x}_{i+1}} \geq \bm{W}_i \bm P\sqbra{\bm{x}_i} + \bm{b}_i, &\quad \paren{i \in\sqbra{L}},\label{eq:sdp_ineq_relu}\\
&\quad \diag{\bm P\sqbra{\bm{x}_{i+1} \bm{x}_{i+1}^{\top}} - \bm{W}_i \bm P\sqbra{\bm{x}_i \bm{x}_{i+1}^{\top}}} - \bm{b}_i \odot \bm P\sqbra{\bm{x}_{i+1}} = \bm{0}  &\quad\paren{i \in\sqbra{L}}, \label{eq:sdp_relax_relu}\\
&\quad\diag{\bm P\left[\bm{x}_i \bm{x}_i^{\top}\right]} - \paren{\bm{l}_i + \bm{u}_i} \odot \bm P\sqbra{\bm{x}_i} + \bm{l}_i \odot \bm{u}_i \leq \bm{0} &\quad \paren{i \in\sqbra{L}}, \label{eq:sdp_relax_bound}\\
&\quad\bm P\left[1\right] = 1,~\bm P \succeq \bm{O}, \label{eq:sdp_semidefinite}
\end{alignat}
\end{subequations}
where we use the same indexings $\bm P\sqbra{\cdot}$ as in~\cite{SDP-IP}.
The constraints~\eqref{eq:relu_qcqp} and \eqref{eq:qcqp} are  reformulated as  linear constraints in \Cref{eq:sdp_ineq_nonneg,eq:sdp_ineq_relu,eq:sdp_relax_relu,eq:sdp_relax_bound}.
Note also that the constraints in \Cref{eq:sdp_semidefinite} are valid constraints when $\bm P=\bm v \bm v^{\top}$.

\subsubsection{
Preprocessing to Remove Inactive Neurons
} \label{sec:preprocessing}

We always apply a popular preprocessing to remove identified inactive neurons with the upper bound $u_{ij} = 0$ before solving the problem~\eqref{prob:qcqp} throughout this paper. We follow the same preprocessing methodology used in prior work~\cite{batten2021efficient} and use the $\alpha$-CROWN~\cite{xu2020fast} to obtain the upper and lower bounds.
However, we emphasize that it is generally impossible to identify and remove all inactive neurons before solving DNN verification. This preprocessing identifies inactive neurons with lightweight verification methods, such as $\alpha$-CROWN~\cite{xu2020fast}. However, these methods are still incomplete verification techniques, i.e., they cannot identify all inactive neurons. We can identify all inactive neurons if we apply complete verification, such as MIP-based verification, but this is almost equivalent to solving the original problem and is generally impractical. 

\section{Interior-Point Vanishing}
\label{sec:interior_point_vanishing}

We investigate the impact of the interior-point vanishing on the numerical stability and optimality of the SDP relaxation problem~\eqref{eq:sdp_relax}. 
We define interior-point vanishing to describe the phenomenon of SDP-based verification's inability to have feasible interior points (i.e., strict feasibility) when the depth of the DNN being verified increases. We first introduce the critical role of strict feasibility when solving SDP problems and then conduct an empirical and theoretical impact analysis on the SDP relaxation problem.

\subsection{Strict Feasibility and Slater's Condition}
\label{sec:slater}
We first introduce the strict feasibility and Slater's condition, which are essential for the strong duality theorem in SDP problems. 
We start this discussion from the standard form of the SDP problem. 
Let $\mathcal S^n$ be the set of $n\times n$ real-valued symmetric matrices. 
Then, the standard form of the primal SDP problem is given as follows:
\begin{subequations} \label{eq:sdp_primal}
\begin{align}
\min_{\bm X} &\quad \tr{\bm C\bm X} \\
\st &\quad
\tr{\bm A_j  \bm X} = b_j~(j = 1, \ldots, m), \\ 
&\quad \bm X \succeq \bm O, 
\end{align}
\end{subequations}
where the variable  is $\bm X\in \mathcal S^n$, and $\bm A_j\in \mathcal S^n$, $b_j\in \mathbb R~(j=1.\dots, m)$, and $\bm C\in \mathcal S^n$ are given parameters.

When the SDP problem~\eqref{eq:sdp_primal} has a feasible solution with $\bm X\succ \bm O$, we say that the problem is \textit{strictly feasible}, and such a feasible solution $\bm X$ is called a \textit{strictly feasible solution} or an \textit{interior feasible solution}.
Strict feasibility is crucial when solving SDP problems, since Slater's condition, a sufficient condition for strong duality, requires strict feasibility~\cite{lourencco2016structural}.

\begin{theorem}[Strong Duality]
	If the primal problem \eqref{eq:sdp_primal} and its dual are both strictly feasible, then both have bounded optimal solutions
    and have the same optimal value.
\end{theorem}

To determine whether the current solution is optimal, the primal-dual interior-point method requires the gap between the objective functions of the primal and dual problems to be sufficiently close to zero. 
However, when the strong duality does not hold, such a gap is not necessarily zero for optimal solutions, and thus, the primal-dual interior-point method may fail to determine the optimality of the current solution.
In addition, \cite{sekiguchi2021perturbation} reports that the lack of strong duality causes serious numerical instability and gives wrong optimal values and solutions when the problem is not strictly feasible. This numerical instability is critical for interior-point methods and can be even more severe for first-order methods~\cite{boyd2011distributed, sun2020sdpnal, dathathri2020enabling}, as the latter cannot exploit second-order Hessian information during optimization. We further discuss this in~\Cref{sec:layer_bound}.

To identify the strict feasibility of the SDP problem~\eqref{eq:sdp_primal}, we design the following verification problem:
\begin{proposition}[Verification of Strict Feasibility]\label{prop:strong_feasibility}
	Consider the following problem: 
	\begin{subequations} \label{eq:sdp_eign}
		\begin{align}
		\max_{\bm X, \lambda} \quad &\lambda \\
		\st 
		\quad & \tr{\bm A_j \cdot \paren{\bm X + \lambda \bm I}}  = b_j ~(j = 1, \dots, m)\\
		&\bm X \succeq \bm O,
		\end{align}
		\end{subequations}
		where $\bm I$ is the identity matrix and $\lambda \in \mathbb R$ is an auxiliary variable. 	
	The original SDP problem~\eqref{eq:sdp_primal} is strictly feasible if and only if the optimal value of the problem~\eqref{eq:sdp_eign} is positive.
\end{proposition}

\begin{proof}
See \Cref{sec:proof_strong_feasibility}.
\end{proof}

We note that the SDP-based relaxation problem~\eqref{eq:sdp_relax} can be equivalently expressed in the standard form~\eqref{eq:sdp_primal} since the problem~\eqref{eq:sdp_primal} is a standard form of SDP problems.
Therefore, we can check the strict feasibility of the problem~\eqref{eq:sdp_relax} by solving its corresponding problem~\eqref{eq:sdp_eign}. \Cref{sec:rewrite_strong_feasibility} describes the standard SDP form of \Cref{prop:strong_feasibility}.

\subsection{Empirical Analysis} \label{sec:empirical_analysis}

We empirically investigate the impact of the interior-point vanishing problem using the strict feasibility verification problem~\eqref{eq:sdp_eign} and demonstrate that the interior-point vanishing becomes more severe as the number of layers $L$ increases. 

\noindent\textbf{Experimental Setup:}
We investigated how frequently the interior-point vanishing problem arises in the problem~\eqref{eq:sdp_relax} and how it depends on the number of layers in NN models by solving the problem~\eqref{eq:sdp_eign}. 
In this experiment, we used two datasets: MNIST~\cite{deng2012mnist} and Fashion-MNIST~\cite{xiao2017fashion}, and trained fully connected ReLU networks with cross-entropy loss and AdaDelta optimizer~\cite{zeiler2012adadelta}. 
For ReLU networks, we set the number of layers $L$ from 2 to 16 and the number of neurons in each layer to 20. 

To construct the problem instances for the problem~\eqref{eq:sdp_eign}, we selected ten images from each dataset and downsampled them to \(5 \ {\rm pixel} \times 5 \ {\rm pixel}\) to reduce the computational load, which is caused by high-precision computation. 
For the instance $\bar{\bm x}$ for verification, we selected the first ten images from each dataset and compressed them to \(5 \ {\rm pixel} \times 5 \ {\rm pixel}\). 
For each instance, we computed the upper and lower bounds $\paren{\bm l_i, \bm u_i}$ of each layer by the $\alpha$-CROWN~\cite{xu2020fast}. We removed all inactivate neurons with $u_{ij} = 0$, similar to prior work~\cite{batten2021efficient}. We then solved the corresponding strict feasibility verification problem~\eqref{eq:sdp_eign} by using the SDPA-GMP~\cite{SDPA-GMP}.
Note that the SDPA-GMP solver employs multi-precision arithmetic in a primal-dual interior point method, making it more reliable than standard double-precision arithmetic for ill-conditioned SDP problem instances. We used the hexadecuple precision (512 bits).

\noindent\textbf{Results and Analysis:}
\Cref{tbl:min_eign} summarizes the results for each $L$.   
For each dataset, we constructed five distinct models with different random seeds, yielding a total of 50 problem instances per dataset (10 images times 5 models).
The column \textit{Solved (\%)} in \Cref{tbl:min_eign} shows the fraction of instances that SDPA-GMP successfully terminated its computation with an optimality guarantee, and the column \emph{Avg. Obj.} represents the average of the optimal values across the solved instances. 
From Table~\ref{tbl:min_eign}, we observe that for both MNIST and Fashion-MNIST, SDPA-GMP consistently terminated its computation for most instances when the DNN depth was less than 10.  
However, as $L$ increased, the \textit{Solved (\%)} decreased, and when $L= 16$, SDPA-GMP failed to finish the computation for all instances in both datasets.
Furthermore, when $L\ge 10$,  \emph{Avg. Obj.} was sufficiently close to zero regardless of the datasets, indicating that the SDP-based verification problem~\eqref{eq:sdp_relax} did not have strictly feasible solutions.
This lack of strict feasibility appears to be a major factor causing SDPA-GMP to fail.  
Overall, these results highlight that the interior-point vanishing problem is likely to occur in the SDP-based DNN verification when $L$ is large.

\noindent\textbf{Interior-Point Vanishing in Benchmark Networks:} 
To further validate our findings, we also evaluated networks commonly used in prior verification research~\cite{salman2019convex, chiu2023tight}, including models trained with different robustification techniques such as dual formulation training, adversarial training, and standard training. Our experiments on these networks (ranging from 2 to 9 layers with 100 neurons per layer) consistently showed interior-point vanishing across all training methodologies, with minimum eigenvalues at or near zero. This confirms that the phenomenon is an inherent property of the SDP relaxation rather than a training artifact, and affects practical networks used in the verification works. Detailed results are in~\Cref{sec:appendix_large_network}.

\begin{table}[t!]
\centering
\setlength{\tabcolsep}{4pt}
\caption{Impact of the interior-point vanishing problem on the different layer depths.
The gray rows indicate that the interior-point vanishing almost always happens as the minimum eigenvalues are almost zero, and even negative due to numerical errors.
}
\begin{tabular}{crrrr}
\toprule
$L$ & \multicolumn{2}{c}{MNIST} & \multicolumn{2}{c}{Fashion-MNIST} \\
\cmidrule{2-5}
 & Solved (\%) & Avg. Obj.   &  Solved (\%)  & Avg. Obj.  \\
\midrule
2 & 98\% & 2.13$\pm$1.93E-05 & 100\% & 5.79$\pm$5.76E-05 \\
4 & 98\% & 1.72$\pm$1.45E-06 & 100\% & 4.93$\pm$5.73E-06 \\
6 & 98\% & 8.06$\pm$5.38E-08 & 98\% & 1.56$\pm$1.12E-07 \\
8 & 98\% & 3.52$\pm$3.47E-09 & 94\% & 4.98$\pm$5.88E-09 \\ \rowcolor{lightgray!30}
10 & 18\% & $-$4.09$\pm$1.70E-10 & 26\% & $-$2.57$\pm$3.11E-10 \\ \rowcolor{lightgray!30}
12 & 2\% & $-$8.31$\pm$2.96E-10 & 4\% & $-$6.90$\pm$2.62E-10 \\ \rowcolor{lightgray!30}
16 & 0\% & $-$1.20$\pm$8.24E-09 & 0\% & $-$9.35$\pm$3.47E-10 \\ 
\bottomrule
\label{tbl:min_eign}
\end{tabular}
\end{table}

\subsubsection{Constraints Related to Interior-Point Vanishing Problem}

To identify the constraints causing the interior-point-vanishing problem, we analyze which constraint in~\eqref{eq:sdp_relax} is significant in the problem's feasibility by removing the constraints one by one. We find that the ReLU equality constraint~\eqref{eq:sdp_relax_relu} and the upper bound $u_i$ in \Cref{eq:sdp_relax_bound} have high sensitivity to the problem feasibility. Details are in \Cref{sec:prelim_analysis}. We will explore the relaxation methods on these constraints in~\Cref{sec:methodology}.

\subsection{Theoretical Analysis}  \label{sec:theoritical_analysis}

We explore the theoretical reasoning behind why the interior-point vanishing problem happens when the depth of DNN increases by examining the structures of the problem~\eqref{eq:sdp_relax}.

\subsubsection{
Existence of Inactive Neurons	
} \label{sec:zero_diag}

We first demonstrate that interior-point vanishing occurs—that is, problem~\eqref{eq:sdp_relax} lacks strictly feasible solutions—when the DNN contains inactive neurons whose outputs are always zero. The following proposition states that a positive semidefinite matrix $\bm P\succeq \bm O$ is not positive definite if it has a diagonal element $\paren{\bm P}_{ii}=0$.

\begin{proposition}\label{prop:zero_diag}
	For a positive semidefinite matrix $\bm P\succeq \bm O$, if there exists a diagonal element $\paren{\bm P}_{ii}=0$, $\lambda_{\min}\paren{\bm P}=0$, that is $\bm P$ is not positive definite.
\end{proposition}
\begin{proof}
See~\cite{Horn2012-lu}, for example.
\end{proof}

According to \Cref{prop:zero_diag}, the minimum eigenvalue of feasible $\bm P$ will always be zero if some diagonal elements of the matrix variable $\bm P$ in the problem~\eqref{eq:sdp_relax} are forced to be zero. In this case, the problem \eqref{eq:sdp_relax} does not have any strictly feasible solutions.  
We explore a situation in which this issue arises within problem~\eqref{eq:sdp_relax}.
Let us consider a situation where an entry of the upper bound vector $\bm  u_i$ is zero for the $i$-th layer, which means that the corresponding neuron is always inactive for all $\bm x_0 \in \{\bm x \mid \|\bm x-\bar{\bm x}\|_\infty\le \rho\}$. 
Let $j$ be the corresponding index of the neuron of the $i$-th layer
with $\paren{\bm u_i}_j=0$.  
In this case, its corresponding lower bound $\paren{\bm l_i}_j \le 0$, and from the constraint~\eqref{eq:sdp_relax_bound} for the $j$-th neuron of the $i$-th layer, we have
\begin{equation*}
\paren{\bm P\sqbra{ \bm x_i \bm x_i^\top}}_{jj} \le \paren*{\paren{ \bm l_i + \bm u_i} \odot  \bm P\sqbra{\bm x_i} + \bm l_i \odot \bm u_i}_j  
\le 0. 
\end{equation*}
Along with the semidefinite constraint $\bm P\succeq \bm O$ and the constraint~\eqref{eq:sdp_relax_bound}, we have $\paren{\bm P\sqbra{\bm x_i \bm x_i^\top }}_{jj}=0$. 
This implies that for all feasible $\bm P$ for the problem~\eqref{eq:sdp_relax}, its minimum eigenvalue is zero. 
Consequently, we see that if there are neurons that are always inactive for all input $\bm x_0\in \bra{\bm x \mid \norm{\bm x-\bar{\bm x}}_\infty \le \rho}$, the problem~\eqref{eq:sdp_relax} does not have any strictly feasible solutions.
As discussed in~\cref{sec:preprocessing}, we cannot identify and remove all inactive neurons before solving DNN verification. Since even a single inactive neuron can cause interior-point vanishing, this indicates that there is no trivial method to address the interior-point vanishing problem.

\subsubsection{Minimum Eigenvalue Bound}

We examine the minimum eigenvalue of feasible $\bm P$ to the problem~\eqref{eq:sdp_relax} and show that the norm of the weight matrix of each layer is crucial for the strict feasibility of the problem~\eqref{eq:sdp_relax}. 
For notational simplicity, let us define a constant matrix horizontally concatenating $\bm b_i$ and $\bm W_i$ as 
	$\widetilde{\bm W}_i = 
	\paren{
		\bm b_i ~ \bm  W_i
	}~\paren{i\in \sqbra{L}}$.

First, we show that the trace of $\bm P\sqbra{\bm x_i \bm x_i^\top}$ for each $i\in \sqbra{L}$ is bounded by a nonnegative constant. 
We formally express this in the following lemma:
\begin{lemma}[Trace Bound Propagation] \label{lem:trace}
Let $T_0 \coloneqq \paren{\norm{\bar{\bm x}}_2+\rho \sqrt{n_0}}^2\ge 0$.
For each $i\in \sqbra{L-1}$, recursively define
\begin{equation*}
	T_{i+1} \coloneqq (1+T_i)\cdot \norm{\widetilde{\bm W}_i}_F^2\quad \paren{i\in \sqbra{L-1}}.	
\end{equation*}
Then, for any feasible solution $\bm P\succeq \bm O$ to the problem~\eqref{eq:sdp_relax}, the following inequality holds:
\begin{equation}\label{eq:trace_bound}
	\tr{\bm P\sqbra{\bm x_i \bm x_i^\top}} \le T_i\quad \paren{i\in \sqbra{L}}.
\end{equation}
\end{lemma}

\begin{proof}
    See Appendix~\ref{lem:trace}.
\end{proof}

While \Cref{lem:trace} focuses on bounding $\tr{\bm{P}\sqbra{\bm{x}_i \bm{x}_i^\top}}$ for each $i\in \sqbra{L}$, we can also derive individual upper bounds for each diagonal element of feasible $\bm{P}$ for the problem~\eqref{eq:sdp_relax} in a similar manner.

\begin{lemma}[Element-wise Bound Propagation]\label{lem:elementwise}
	Let $T_i$ be the constant defined in \Cref{lem:trace}. 
	Then, for each $i\in \sqbra{L}$, 	
	the $j$-th diagonal element of $\paren{\bm{P}\sqbra{\bm{x}_i \bm{x}_i^\top}}_{jj}$ is bounded as
	\begin{equation*}
		\paren{\bm{P}\sqbra{\bm{x}_{i+1} \bm{x}_{i+1}^\top}}_{jj} \le \paren{1+T_{i}} \cdot \norm{\bm{\widetilde{W}}_{i}\paren{j,\colon}}_2^2, 
	\end{equation*}	
	where $\bm{\widetilde{W}}_{i}\paren{j,\colon}$ is the $j$-th row vector of $\bm{\widetilde{W}}_i$.
\end{lemma}

We can show \Cref{lem:elementwise} in the same way as \Cref{lem:trace} by focusing on the $j$-th diagonal element of $\bm{P}\sqbra{\bm{x}_i \bm{x}_i^\top}$, and we omit the proof here.

From \Cref{lem:elementwise} and the fact that $\lambda_{\min}\paren{\bm{P}} \le \min_{j} \paren{\bm{P}}_{jj}$ for $\bm{P} \succeq \bm{O}$, we finally provide an upper bound of the minimum eigenvalue of feasible $\bm{P} \succeq \bm{O}$ to the problem~\eqref{eq:sdp_relax}.
\begin{theorem}[Minimum Eigenvalue Bound]\label{thm:min_eign}
	Let $T_i$ be the constant defined in \Cref{lem:trace}. 
	Then, for all feasible solutions $\bm{P}$ to the problem~\eqref{eq:sdp_relax}, it follows that
	\begin{equation*}
		\lambda_{\min}\paren{\bm{P}} \le \min_{i\in \sqbra{L}}\bra*{\min_{j=1,\dots,n_i} \bra*{\paren{1+T_{i}} \cdot \norm{\bm{\widetilde{W}}_{i}\paren{j,\colon}}_2^2}}.
	\end{equation*}
\end{theorem}

\Cref{thm:min_eign} highlights that the minimum norm of the extended row vector $\bm{\widetilde{W}}_{i}\paren{j,\colon}$ plays a crucial role in controlling the possible minimum eigenvalue of feasible $\bm{P}$ to the problem~\eqref{eq:sdp_relax}.
Specifically, \Cref{thm:min_eign} indicates that when there exists a single neuron $j$ in some layer $i\in\sqbra{L}$ with a small $\norm{\bm{\widetilde{W}}_{i}\paren{j,\colon}}_2$, the minimum eigenvalue of any feasible $\bm{P}$ is constrained to be close to zero. 
\emph{Consequently, the presence of even a single such neuron can trigger the interior-point vanishing problem in SDP-based DNN verification.}

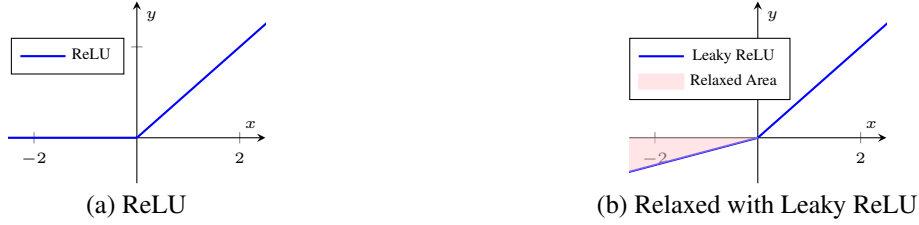
\begin{figure}[t]
    \centering
    \begin{minipage}[b]{0.42\linewidth}
        \centering
        \begin{tikzpicture}
            \tikzstyle{every node}=[font=\tiny]
            \begin{axis}[
                xmin=-2.5, xmax=2.5,
                ymin=-1, ymax=3,
                axis lines=middle,
                xlabel={$x$},
                ylabel={$y$},
                xtick={-4,-2,0,2,4},
                ytick={-4,-2,0,2,4},
                legend style={at={(0.01,.8)},anchor=north west},
                width=5cm,
                height=4cm
            ]
                \addplot[blue, thick, domain=-4:4, samples=200] {max(0,x)};
                \addlegendentry{ReLU}
            \end{axis}
        \end{tikzpicture}\\ 
        (a) ReLU
    \end{minipage}
    \hfill %
    \begin{minipage}[b]{0.42\linewidth}
        \centering
        \begin{tikzpicture}
            \tikzstyle{every node}=[font=\tiny]
            \begin{axis}[
                xmin=-2.5, xmax=2.5,
                ymin=-1, ymax=3,
                axis lines=middle,
                xlabel={$x$},
                ylabel={$y$},
                xtick={-4,-2,0,2,4},
                ytick={-4,-2,0,2,4},
                legend style={at={(0,.8)},anchor=north west},
                width=5cm,
                height=4cm
            ]
                \addplot[blue, thick, domain=-4:4, samples=200] {x > 0 ? x : 0.3*x};
                \addlegendentry{Leaky ReLU}
                
                \addplot[red!20, opacity=0.5, domain=-4:0, samples=200, fill=red!20, fill opacity=0.5, area legend,
                ]
                    {0.3*x} \closedcycle;
                \addlegendentry{Relaxed Area}
            \end{axis}
        \end{tikzpicture}\\ 
        (b) Relaxed with Leaky ReLU
    \end{minipage}
    \hfill
    \begin{minipage}[b]{0.005\linewidth}
    \end{minipage}
    \caption{Relaxed area of our LeakySDP by Leaky ReLU function.}
    \label{fig:leaky_relu}
\end{figure}
\section{Methodology} \label{sec:methodology}

To address the interior-point vanishing problem, we designed five methods to relax the constraints to obtain more feasible interior points. Based on the observations found in~\Cref{sec:interior_point_vanishing}, we mainly explore the further relaxations regarding the ReLU equality constraint \eqref{eq:sdp_relax_relu} and the upper bound $\bm u_i$ in \Cref{eq:sdp_relax_bound}. 

\subsection{Relaxation of Equality Constraints: \textit{$\varepsilon$-SDP}}
We first simply relax the ReLU equality constraint~\eqref{eq:sdp_relax_relu} by allowing a tolerance $\varepsilon\ge 0$ as follows:
\begin{equation*}
\begin{alignedat}{3}
\min_{\bm P}~&\quad \bm{c}^{\top} \bm P\sqbra{\bm{x}_L} + c_0 \notag\\
\st~&\quad \text{\Cref{eq:sdp_ineq_nonneg,eq:sdp_ineq_relu,eq:sdp_relax_bound,eq:sdp_semidefinite}, and}\notag \\ 
&\quad -\varepsilon \cdot \bm 1_{n_i} \leq \diag{\bm{P}\sqbra{\bm{x}_{i+1} \bm{x}_{i+1}^{\top}} - \bm{W}_i \bm{P}\sqbra{\bm{x}_i \bm{x}_{i+1}^{\top}} }  - \bm{b}_i \odot \bm{P}\sqbra{\bm{x}_{i+1}} \leq \varepsilon\cdot \bm 1_{n_i} &\quad \paren{i \in \sqbra{L}}. \notag
\end{alignedat}  
\end{equation*}

We call this new relaxed incomplete verification \textit{$\varepsilon$-SDP}. This relaxation is simple but straightforwardly relaxes the ReLU equality constraint.

\subsection{SDP Relaxation with Leaky ReLU: \textit{LeakySDP}}
We also relax the ReLU constraint \eqref{eq:sdp_relax_relu} with the Leaky ReLU function, which is defined as:
\[
\LeakyRelu{x} = 
\begin{cases} 
x & \paren{x \geq 0}, \\
\alpha \cdot x & \paren{x < 0},
\end{cases}
\]
where \( \alpha \in (0, 1) \) is a small positive constant. Similar to the ReLU, we can describe the Leaky ReLU with linear constraints on $\bm P$.

Since the DNN to be verified uses the ReLU, we relax the ReLU activation with the Leaky ReLU as illustrated in \Cref{fig:leaky_relu}. We allow the verification problem to take the value in the area between the Leaky ReLU and the x-axis when $x \geq 0$. As it still contains the ReLU function part, this operation relaxes the original problem.
Formally, the constraints~\eqref{eq:sdp_ineq_nonneg},~\eqref{eq:sdp_ineq_relu}, and~
\eqref{eq:sdp_relax_relu} are replaced with the inequalities and equality constraints associated with the Leaky ReLU function as follows:
\begin{equation*}
\begin{alignedat}{3}
\min_{\bm P}~&\quad  \bm{c}^{\top} \bm P\sqbra{\bm{x}_L} + c_0 \\
\st~&\quad\text{Eqs.~(\ref{eq:sdp_relax_bound}), (\ref{eq:sdp_semidefinite}), and}\\
&\quad  \bm{P}\sqbra{\bm{x}_{i+1}} \geq\alpha\paren*{\bm W_i\bm{P}\sqbra{\bm{x}_i}+\bm{b}_i}&\quad  \paren{i \in \sqbra{L}},\\ 
&\quad 
  \bm{P}\sqbra{\bm{x}_{i+1}} \geq \bm{W}_i\bm{P}\sqbra{\bm{x}_i} +\bm{b}_i&\quad  \paren{i \in \sqbra{L}},\\
&\quad  \operatorname{diag}\paren*{\bm{P}\sqbra{\bm{x}_{i+1} \bm{x}_{i+1}^{\top}}
  -\bm{W}_i \,\bm{P}\sqbra{\bm{x}_i \bm{x}_{i+1}^{\top}}}-\bm{b}_i \odot \bm{P}\sqbra{\bm{x}_{i+1}}
  \leq\bm{0}&\quad  \paren{i \in \sqbra{L}}.
\end{alignedat}   
\end{equation*}

We refer to this formulation as \textit{LeakySDP}.

\subsection{Diagonal Scaling: \textit{D-Scale}}
We adopt a common approach to solve an ill-conditioned SDP, called scaling~\cite{wright2006numerical}, which aims to improve the condition number of the matrices involved in SDP to reduce sensitivity to numerical errors. As finding the optimal scaling is almost the same as solving the original SDP, several heuristic approaches have been designed. 
Diagonal scaling~\cite{gao2023scalable} is one of the most commonly used scaling methods due to its low computational overheads. Specifically, the diagonal scaling can be written as follows:
\begin{proposition}
Let $\bm D\succ \bm O$ be a diagonal matrix. 
Then, the standard primal SDP~\eqref{eq:sdp_primal} can be equivalently transformed as follows:
\begin{align*}
\min_{\bm X} &\quad \tr{\paren{\bm{D} \bm{C} \bm{D}} \cdot\paren{\bm{D}^{-1} \bm{X} \bm{D}^{-1}}}\\ 
\st &\quad \tr{\paren{\bm{D} \bm{A}_j \bm{D}}\cdot \paren{ \bm{D}^{-1} \bm{X} \bm{D}^{-1}} }= b_j \quad (j = 1, \ldots, m), \\
&\quad \bm{D}^{-1} \bm{X} \bm{D}^{-1} \succeq \bm{O}.
\end{align*}
\end{proposition}
As a scaling diagonal matrix, we use $\bm D$ whose diagonal elements $\diag{\bm D} = \paren{1, \bm u_1^\top, \bm u_2^\top, \ldots, \bm u_L^\top}^\top$. We note that division by zero does not occur since we remove neurons with  $u_{ij} \leq 0$ in the preprocessing as discussed in~\Cref{sec:preprocessing}.

\subsection{Neural Network Weight Scaling: \textit{W-Scale}}
\label{sec:w-scale}

We further design a new scaling method by leveraging the findings in~\Cref{thm:min_eign}, which shows that the minimum eigenvalue of any feasible solution \(\bm P\) to the SDP relaxation is constrained by the smallest norm of the extended row vector \(\widetilde{\bm{W}}_{i}\paren{j,\colon}\). 
Since a neuron \(j\) in some layer \(i\) has a particularly small norm \(\|\widetilde{\bm{W}}_{i}(j,\colon)\|_2\) triggers the interior-point vanishing problem, we rescale the parameters in each layer to avoid excessively small row norms. 
Let $\check{w}_i = \min_{j=1,\dots,n_i}\|\widetilde{\bm{W}}_i(j,\colon)\|_2$ be the minimum $L_2$ norm of the row vectors of $\widetilde{\bm{W}}_i$. 
According to $\check{w}_i$, we scale the weight matrix $\widetilde{\bm{W}}_i$ of intermediate layers $i\in \sqbra{L-1}$ as $
\widetilde{\bm{W}}'_i
\coloneqq \widetilde{\bm{W}}_i/\check{w}_i ~\paren{i \in \sqbra{L-1}}, 
$
and for the final layer as 
$
\widetilde{\bm{W}}'_L
\coloneqq
\widetilde{\bm{W}}_L\cdot 
\prod_{i=0}^{L-1}\check{w}_i.    
$
Note that this scaling makes the final output given by the scaled network the same as the original one.
We refer to this approach as \emph{W-Scale}.

\begin{table}[t]
    \centering
    \caption{Comparison of success rates for SDP verification problems solved by our five proposed methods and two prior methods.} \label{tab:5-1}
    \begin{tabular}{c|ccccc|cc}
        \toprule
        $L$ & \multicolumn{5}{c|}{Proposed Methods} & \multicolumn{2}{c}{Prior Methods} \\
        \cmidrule{2-8}
         & \(\varepsilon\)-SDP & LeakySDP & B-Remove  & D-Scale & W-Scale & LayerSDP & SDP-IP \\
        \midrule
        2  & \textbf{100\%} & \textbf{100\%} & \textbf{100\%} & \textbf{100\%} & \textbf{100\%} & \textbf{100\%} & \textbf{100\%} \\
        4  & \textbf{100\%} & \textbf{100\%} & \textbf{100\%} & \textbf{100\%} & \textbf{100\%} & \textbf{100\%} & \textbf{100\%} \\
        6  & \textbf{100\%} & 98\%  & \textbf{100\%} & \textbf{100\%}  & \textbf{100\%}  & \textbf{100\%}  & \textbf{100\%}  \\
        8  & \textbf{100\%} & 82\%  & \textbf{100\%} & 88\%  & 88\%  & 82\%  & 98\%  \\
        10 & \textbf{100\%}  & 4\%   & \textbf{100\%} & 8\%   & 8\%   & 2\%  & 14\%   \\
        12 & 78\%  & 0\%   & \textbf{100\%} & 0\%   & 0\%   & 0\%   & 0\%   \\
        16 & 24\%   & 0\%   & \textbf{66\%} & 0\%   & 0\%   & 0\%   & 0\%   \\
        \bottomrule
    \end{tabular}
    \label{tab:small_diff_comparison}
\end{table}

\begin{figure}[t]
    \centering
    \includegraphics[width=.7\linewidth]{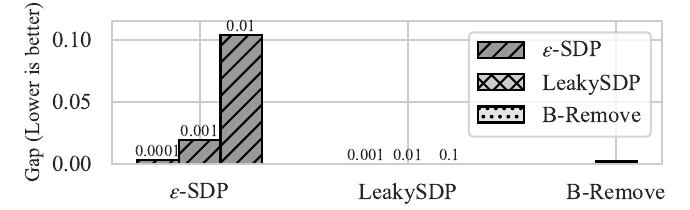}
    \caption{Average gap between the optimal objective values of our methods and the LayerSDP's optimal. The numbers above each bars are $\varepsilon$ for $\varepsilon$-SDP and $\alpha$ for LeakySDP.}
    \label{fig:LayerSDP_gap}
\end{figure}
\subsection{Removal of Bound Constraints: \textit{B-Remove}}

To relax the upper bound constraints highly related to the interior-point vanishing, we simply remove the constraints from the intermediate layers since the SDP relaxation of ReLU does not require such bounds as shown in \Cref{sec:sdp_relax} and Appendix~\ref{sec:prelim_analysis}. The bounds are introduced as valid constraints, which typically help higher converging in many applications. We refer to this approach as \emph{B-Remove}.

\section{Evaluation}\label{sec:evaluation}

We evaluate our five methods designed in~\Cref{sec:methodology}, regarding the capability to address the interior-point vanishing problem and the relaxation quality.

\subsection{Experimental Setup}

We trained 35 fully connected DNN models on the MNIST dataset: 
We set the number of layers $L\in \{2, 4,\dots , 16\}$ and tried 5 different seeds for each $L$. 
The number of neurons in each layer was fixed to $n_i=20~\paren{i=1,\dots,L}$. 
We use the original resolution, i.e., \(28 \ {\rm pixel} \times 28 \ {\rm pixel}\), and use the normal version of SDPA~\cite{yamashita2010high}, not the high-precision SDPA-GMP as used in~\Cref{sec:interior_point_vanishing}, to handle the original image size. SDPA may introduce worse numerical stability, but this can allow us to conduct larger-scale experiments and more clearly observe the improved numerical stability from our designed methods. We discuss the choice of SDP solver in~\Cref{sec:sdp_solver}. 
We judge that an SDP problem is successfully solved if the primal-dual objective gap is lower than $10^{-6}$.
We randomly selected 10 images from MNIST and randomly chose targets for the labels. 
We compare our proposed methods, denoted as \(\varepsilon\)-SDP (with parameter \(\varepsilon = 0.01\)) and LeakySDP (with parameter \(\alpha = 0.01\)), against prior SDP-based verification methods: SDP-IP~\cite{SDP-IP}, the original SDP-based verification method, and LayerSDP~\cite{batten2021efficient}, an improved SDP-based verification method that utilizes the sparse structure in the layer-wise structure of DNN and the constraints used in the LP-based verification~\cite{ehlers2017formal}.

\subsection{Evaluation on Interior-Point Vanishing}
\label{sec:eval_ipv}

We evaluate whether our methods can effectively address the interior-point vanishing problem. 
\Cref{tab:small_diff_comparison} shows the success rates for SDP verification problems solved by our five proposed methods and two existing methods. 
As shown, the existing methods (LayerSDP and SDP-IP) suffer from this problem starting at \(L=8\), and fail almost entirely at \(L=10\), as discussed in~\Cref{sec:empirical_analysis}.

In contrast, our \(\varepsilon\)-SDP and \textbf{B-Remove} show substantial improvements. 
These two methods directly address the constraints we identified as a potential source of the interior-point vanishing problem in~\Cref{sec:theoritical_analysis} and \Cref{sec:prelim_analysis}. 
Notably, B-Remove achieves a surprisingly large improvement despite its simplicity—just removing upper and lower bounds. 
We further discuss this result in~\Cref{sec:layer_bound}.

The other three methods (LeakySDP, D-Scale, and W-Scale), however, do not significantly improve the success rate. 
For LeakySDP, it can be noted that while it relaxes the part where the input of the ReLU activation function is non-positive in the negative direction, it has little effect on the positive part. 
Due to this property, the effectiveness of the method is limited when positive activations are dominant.  
In addition, nodes with small absolute values may contribute significantly to the primal-dual gap, but applying Leaky-ReLU to such nodes results in only minor changes, limiting the method’s impact.
D-Scale and W-Scale aimed to improve the ill-condition of the SDP by scaling, which is commonly used to improve the convergence of SDP. 
However, our experiments suggest that such generic approaches were not effective for addressing the interior-point vanishing problem.

\subsection{Evaluation on Relaxation Quality}
\label{sec:eval_quality}

We then evaluate the relaxation quality of our methods. 
Since all our proposed methods relaxed the problem more than SDP-IP and LayerSDP,  our methods generally have less verification power than these existing methods, as long as the existing methods can solve the problem. 
To quantify the verification power losses, we calculate the gap between the optimal objective values obtained by our methods and LayerSDP, which has a tighter or at least eaqual relaxation than SDP-IP. 
A small objective gap implies that although our methods yield looser relaxations, they still maintain comparable verification power to LayerSDP.

\Cref{fig:LayerSDP_gap} shows the average gap between the optimal objective function values of our methods and the LayerSDP's ones. 
The average is only calculated over the instances for which all methods  converge to their optimal. \(\varepsilon\)-SDP shows relatively high verification power loss since it directly relaxes the ReLU constraints, although \(\varepsilon\)-SDP can mitigate the interior-point vanishing problem as in~\Cref{sec:eval_ipv}. LeakySDP does not have significant power losses. Considering LeakySDP shows almost similar performance to LayerSDP in~\Cref{sec:eval_ipv}, the relaxation in the negative area does not look effective. In contrast, B-Remove does not cause a major verification power loss. This is because the SDP relaxation of ReLU does not require the upper and lower bounds other than for the input layer as shown in \Cref{eq:qcqp}, i.e., $\bm l_i$ and $\bm u_i$ for $i > 0$. The lower and upper bounds of the input layer describe the attacker's norm budget $\rho$ in \Cref{prob:verification}, not used to relax ReLU. Nevertheless, prior approaches adopt each layer's upper and lower bounds by following the same custom in the traditional LP-based verification~\cite{ehlers2017formal}. We further discuss this in~\Cref{sec:layer_bound}.

\section{Discussion}\label{sec:discussion}

\subsection{Backfire of Layer Bounds} \label{sec:layer_bound}

We find that the upper and lower bounds of each layer, i.e., $\bm l_i$ and $\bm u_i$ for $i > 0$ even amplify the interior-point vanishing problem (\Cref{sec:eval_ipv}) even though they does not have a major impact on improving the relaxation quality (\Cref{sec:eval_quality}). For the SDP-based DNN verification, the upper and lower bounds are not mandatory to formulate the problem. Nevertheless, prior approaches adopt them because the SDP formulation is extended from the traditional LP-based verification~\cite{ehlers2017formal}, which requires the bounds to form the triangle area consisting of the three points at the origin, lower bound, and upper bound, to relax the ReLU function. Generally, more valid linear constraints are considered to be beneficial in improving the verification quality, e.g., LayerSDP introduces constraints inspired by the triangle area in the LP-based verification. However, this assumption does not always hold for SDP. Unnecessary valid constraints may remove the feasible interior points and lead to numerical instability. The increase in the number of linear constraints also introduces additional computational costs. Even if it works for a toy shallow model, we recommend validating the effectiveness on the deeper model at least with 10 layers, ideally more than 20 layers.

\subsection{Applicability of Non-Interior-Point Methods} \label{sec:non_ipm_solver}

We predominantly used the interior-point method, more specifically, the primal-dual interior-point path-following method, to solve the SDP problems in this study. While prior work~\cite{dathathri2020enabling} demonstrates that their first-order method outperforms interior-point methods in DNN verification,  we cannot employ first-order methods~\cite{boyd2011distributed, sun2020sdpnal, dathathri2020enabling} for our research as they cannot provide theoretical guarantees of optimality. Since first-order methods depend on heuristic stopping criteria and step size adjustments, we cannot conclusively determine the existence of interior-point vanishing, even if the first-order methods say they converge to some solutions. Furthermore, the interior-point methods have generally higher numerical stability and accuracy compared to the first-order methods, as the interior-point method can utilize second-order Hessian information~\cite{lin2021admm, tu2014practical}, and thus the first-order methods should suffer from the lack of feasible interior points more significantly. We note that the numerical instability introduced by the lack of feasible interior points is not a specific problem in the interior-point method, but a more fundamental issue in the SDP problem itself. As our research primarily aims to identify and diagnose the interior-point vanishing problem in SDP-based DNN verification rather than directly developing a practical method, we only used the interior-point method.

\subsection{Choice of SDP Solver} \label{sec:sdp_solver}

We use SDPA~\cite{yamashita2012latest} and its high-precision version SDPA-GMP~\cite{SDPA-GMP} with multiple-precision arithmetic. The main reason for the choice is the availability of multiple-precision arithmetic. We also tried MOSEK~\cite{mosek} solver, but we finally decided not to use MOSEK due to our specific needs. MOSEK has a parameter, intpntCoTolNearRel, which allows for violating the constraint at a certain level. We observed a case where the equality constraints were not satisfied to the order of 10e-6. Furthermore, MOSEK's parameter semidefiniteTolApprox allows for violations of the semidefinite constraint, potentially permitting negative eigenvalues. In contrast, we confirmed that SDPA and SDPA-GMP strictly enforce semidefinite constraints based on their open-source implementations. Since MOSEK is closed-source, we could not diagnose these issues by ourselves, which ultimately led us to choose not to use MOSEK for this study. 

\subsection{Applicability of Facial Reduction}

This study identifies the interior-point vanishing problem, describing the fact that the SDP-based DNN verification suffers from a lack of feasible interior points (i.e., strict feasibility) when the depth of the DNN increases. Meanwhile, the numerical instability caused by the lack of strict feasibility is a well-known problem in the SDP area~\cite{lourencco2016structural, sekiguchi2021perturbation}. One of the most fundamental approaches to addressing the issue is facial reduction~\cite{permenter2018partial, hu2023facial},  which aims to reduce the size of the original SDP problem, ultimately transforming it into another equivalent problem with a lower dimension where strictly feasible solutions are available. Facial reduction can be iteratively applied until a feasible interior point exists~\cite{waki2013facial}. However, facial reduction techniques do not always guarantee the achievement of a problem with strict feasibility. Furthermore, applying facial reduction itself requires iteratively solving SDP problems. This procedure may even introduce another numerical instability. As facial reduction is unlikely to yield a practical level,
we did not explore it in this research. Practical facial reduction techniques specializing in the DNN verification problem structure can be a potential future research.
\section{Conclusion}\label{sec:conclusion}

In this study, we discover a fundamental challenge in the convergence of SDP-based DNN verifiers, interior-point vanishing, and we are the first to evaluate the impact of strict feasibility on SDP-based DNN verification against the DNN models with practical depth, which is not covered by prior work. We conduct a theoretical and empirical analysis of the interior-point vanishing and find the potential root causes of the problem.
We newly designed five approaches to improve the feasibility of SDP-based verification problems and demonstrate that our methods can successfully solve 88\% of the problems that could not be solved by existing methods, accounting for 41\% of the total.
We further find that the valid constraints for the lower and upper bounds for each ReLU unit are traditionally inherited from prior work without solid reasons, but are actually not only not beneficial but also even harmful to the problem's feasibility. This work provides valuable insights into the fundamental challenges of SDP-based DNN verification, contributing to the development of more reliable and secure systems with DNNs.

\bibliography{main}
\bibliographystyle{plain}

\appendix
\section*{Appendix}

\section{Detailed Proofs} \label{sec:proofs}

\subsection{Proof of \Cref{prop:strong_feasibility}}
\label{sec:proof_strong_feasibility}

\begin{proof}
Let $(\bm X^*, \lambda^*)$ be an optimal solution to the problem~\eqref{eq:sdp_eign}. 
Suppose that $\lambda^* > 0$, we have $\bm X^* + \lambda^* \bm I \succ \bm O$. 
Also, $\bm X^* + \lambda^* \bm I$ is a feasible solution to the problem~\eqref{eq:sdp_primal}, which means the original problem is strictly feasible.

On the other hand, suppose that the original problem has a strictly feasible solution $\bm X'\succ \bm O$. 
Since $\bm X' \succ \bm O$, we have $\bar{\lambda}=\lambda_{\min}(\bm X') > 0$.  
By construction, $\bar{\bm X} = \bm X' - \bar{\lambda}\bm I \succeq \bm O$ since all eigenvalues of $\bar{\bm X}$ remain nonnegative.  
Then, since we have $\tr{\bm A_j  \paren{\bar{\bm X} + \bar{\lambda} \bm I}} = \tr{\bm A_j  \bm X'}= b_j ~(j = 1, \dots, m)$ and $\bar{\bm X} \succeq \bm O$,  
$\paren{\bar{\bm X}, \bar{\lambda}}$ is a feasible solution to the problem~\eqref{eq:sdp_eign}, and its objective value is $\lambda_{\min}\paren{\bm X'}>0$.
Thus, we see that the optimal value of the problem~\eqref{eq:sdp_eign} is positive, which completes the proof.
\end{proof}

\subsection{Proof of \Cref{lem:trace}}
\label{sec:proof_trace}

\begin{proof}
We prove this by induction for $i\in \sqbra{L}$. 
Let us consider the initial case $i=0$. 
Since $\bm P\succeq \bm O$, we have
\begin{equation*}
	\begin{pmatrix}
		1 & \bm P\sqbra{\bm x_0^\top } \\
		\bm P\sqbra{\bm x_0} &\bm  P\sqbra{\bm x_0 \bm x_0^\top}
	\end{pmatrix} \succeq \bm O,
\end{equation*}
which implies 
\begin{equation*}
	\norm{\bm P\sqbra{\bm x_0}}^2_2\le \tr{\bm P\sqbra{\bm x_0 \bm x_0^\top}}. 
\end{equation*}

Recalling $\bm{l}_0=\bar{\bm{x}}-\varepsilon \bm{1}$ and $\bm{u}_0=\bar{\bm{x}}+\varepsilon \bm{1}$, from the constraint~\eqref{eq:sdp_relax_bound}, we have
\begin{align*}
	\tr{\bm{P}\sqbra{\bm{x}_0 \bm{x}_0^\top}} & = \paren{\bm{l}_0+\bm{u}_0}^\top \bm{P}\sqbra{\bm{x}_0} - \bm{l}_0^\top  \bm{u}_0\\
	&\le 2\norm{\bar{\bm{x}}}_2\norm{\bm{P}\sqbra{\bm{x}_0}}_2 - \norm{\bar{\bm{x}}}_2^2 + \varepsilon^2\cdot n_0\\
	&\le 2\norm{\bar{\bm{x}}}_2\sqrt{\tr{\bm{P}\sqbra{\bm{x}_0 \bm{x}_0^\top}}}  - \norm{\bar{\bm{x}}}_2^2 + \varepsilon^2\cdot n_0.
\end{align*} 
Rearranging the above inequality gives 
\begin{equation*}
	\paren*{\sqrt{\tr{\bm{P}\sqbra{\bm{x}_0 \bm{x}_0^\top}}} - \norm{\bar{\bm{x}}}_2}^2 \le  \varepsilon^2\cdot n_0.	
\end{equation*}
Thus, by taking the square root of both sides, we obtain 
\begin{equation*}
	\tr{\bm{P}\sqbra{\bm{x}_0 \bm{x}_0^\top}} \le \paren{\norm{\bar{\bm{x}}}_2+\varepsilon \sqrt{n_0}}^2 = T_0,
\end{equation*}
which ensures that \Cref{eq:trace_bound} holds for $i=0$.

Next, let $i\in \sqbra{L-1}$, and assuming that \Cref{eq:trace_bound} holds for all indices up to and including $i$, we show that it also holds for $i+1$.
Let $\bm{P} \succeq \bm{O}$ be any feasible solution for the problem~\eqref{eq:sdp_relax}. 
Then, the following $3\times 3$ principal submatrix is positive semidefinite:
\begin{equation}\label{eq:3x3}
	\begin{pmatrix}
		1 & \bm{P}\sqbra{\bm{x}_i^\top} & \bm{P}\sqbra{\bm{x}_{i+1}^\top} \\
		\bm{P}\sqbra{\bm{x}_i} & \bm{P}\sqbra{\bm{x}_i \bm{x}_i^\top} & \bm{P}\sqbra{\bm{x}_i \bm{x}_{i+1}^\top} \\
		\bm{P}\sqbra{\bm{x}_{i+1}} & \bm{P}\sqbra{\bm{x}_{i+1} \bm{x}_i^\top} & \bm{P}\sqbra{\bm{x}_{i+1} \bm{x}_{i+1}^\top}
	\end{pmatrix} \succeq \bm{O}.
\end{equation}
From the induction hypothesis, we have $\tr{\bm{P}\sqbra{\bm{x}_i \bm{x}_i^\top}}\le T_i$, which implies $\lambda_{\max}\paren{\bm{P}\sqbra{\bm{x}_i \bm{x}_i^\top}}\le T_i$. 
Thus, by replacing the first $2\times 2$ principal submatrix in the left-hand side of~\eqref{eq:3x3} with $\paren{1+T_i}\bm{I}$, we obtain
\begin{equation}\label{eq:3x3_2}
	\paren*{
		\begin{array}{cc|c}
                & & \bm{P}\sqbra{\bm{x}_{i+1}^\top} \\
                \multicolumn{2}{c|}{\smash{\raisebox{.5\normalbaselineskip}{$(1+T_i)\bm{I}$}}} & \bm{P}\sqbra{\bm{x}_i \bm{x}_{i+1}^\top}  \\  \hline
			\bm{P}\sqbra{\bm{x}_{i+1}} & \bm{P}\sqbra{\bm{x}_{i+1} \bm{x}_i^\top} & \bm{P}\sqbra{\bm{x}_{i+1} \bm{x}_{i+1}^\top}
		\end{array}
	}\succeq \bm{O}.
\end{equation}

Let us define the concatenated matrix as 
\begin{equation*}\label{eq:extended}
	\bm{\widetilde{P}}_{i+1} = \begin{pmatrix}
		\bm{P}\sqbra{\bm{x}_{i+1}^\top} \\
		\bm{P}\sqbra{\bm{x}_i \bm{x}_{i+1}^\top}
	\end{pmatrix}.
\end{equation*} 
By taking the Schur complement of $\paren{1+T_i}\bm{I}$ of the left-hand side in~\cref{eq:3x3_2}, we have
\begin{equation*}
	\bm{P}\sqbra{\bm{x}_{i+1} \bm{x}_{i+1}^\top} - \frac{1}{1+T_i} \bm{\widetilde{P}}_{i+1}\bm{\widetilde{P}}_{i+1}^\top \succeq \bm{O},
\end{equation*}
which implies
\begin{equation}\label{eq:trace-bound}
	\tr{\bm{P}\sqbra{\bm{x}_{i+1} \bm{x}_{i+1}^\top}} \ge \frac{1}{1+T_i}\norm{\bm{\widetilde{P}}_{i+1}}_F^2.
\end{equation}

On the other hand, from the constraint~\eqref{eq:sdp_relax_relu} and the definitions of $\bm{\widetilde{P}}_{i+1}$ and $\bm{\widetilde{W}}_{i}$, we have
\begin{align}
	\tr{\bm{P}\sqbra{\bm{x}_{i+1} \bm{x}_{i+1}^\top}}  
		&= \tr{\bm{\widetilde{W}}_{i} \bm{\widetilde{P}}_{i+1}}\notag \\
		&\le \norm{\bm{\widetilde{W}}_{i}}_F\norm{\bm{\widetilde{P}}_{i+1}}_F \label{eq:trace-bound-2}.
\end{align}
Combining~\eqref{eq:trace-bound} and~\eqref{eq:trace-bound-2}, we derive
\begin{equation*}
	\norm{\bm{\widetilde{P}}_{i+1}}_F \le (1+T_i)\norm{\bm{\widetilde{W}}_{i}}_F, 
\end{equation*}
which leads to 
\begin{equation*}
	\tr{\bm{P}\sqbra{\bm{x}_{i+1} \bm{x}_{i+1}^\top}} \le (1+T_i)\norm{\bm{\widetilde{W}}_{i}}_F^2,
\end{equation*}
which ensures that \Cref{eq:trace_bound} holds for $i+1$ and completes the proof.
\end{proof}

\section{Conversion~\Cref{prop:strong_feasibility} into Standard SDP Form} \label{sec:rewrite_strong_feasibility}

We introduce nonnegative variables \(\lambda_1,\lambda_2 \ge 0\) and decompose
\begin{equation*}
  \lambda \;=\; \lambda_1 \;-\;\lambda_2.
\end{equation*}
Using this decomposition, the objective function becomes one of maximizing 
\(\lambda_1 - \lambda_2\). We then gather the variable matrices into the following block-diagonal matrix:
\begin{equation*}
  \widetilde{\bm{X}}
  \;=\;
  \begin{pmatrix}
    \bm{X} & 0 & 0 \\
    0 & \lambda_1 & 0 \\
    0 & 0 & \lambda_2
  \end{pmatrix}
  \quad (\widetilde{\bm{X}} \succeq \bm O),
\end{equation*}
thereby simultaneously representing 
\(\bm{X} \succeq \bm O,\ \lambda_1 \ge 0,\ \lambda_2 \ge 0.\)

Next, to express the objective function \(\lambda_1 - \lambda_2\), we can define the matrix 
\(\widetilde{\bm{C}}\) as follows:
\begin{equation*}
  \widetilde{\bm{C}}
  \;=\;
  \begin{pmatrix}
    0 & 0 & 0 \\
    0 & 1 & 0 \\
    0 & 0 & -1
  \end{pmatrix}
  \quad \Longrightarrow \quad
  \tr{\widetilde{\bm{C}} \widetilde{\bm{X}}}
  \;=\;
  \lambda_1 \;-\;\lambda_2.
\end{equation*}

Meanwhile, each constraint
\begin{equation*}
  \tr{\bm{A}_j\cdot  \paren{\bm{X} + \paren{\lambda_1 - \lambda_2}\bm{I}}} 
  \;=\; b_j
\end{equation*}
can be written as
\begin{equation*}
  \tr{\bm{A}_j \bm{X}}
  \;+\;
  \paren{\lambda_1 - \lambda_2}\,\tr{\bm{A}_j  \bm{I}}
  \;=\;
  b_j,
\end{equation*}
where \(\paren{\bm{A}_j  \bm{I}} = \tr{\bm{A}_j}\) is a constant. 
Defining
\begin{equation*}
  \widetilde{\bm{A}}_j 
  =
  \begin{pmatrix}
    \bm{A}_j & 0 & 0 \\
    0 & \tr{\bm{A}_j} & 0 \\
    0 & 0 & -\,\tr{\bm{A}_j}
  \end{pmatrix} \quad \Longrightarrow \quad
  \tr{\widetilde{\bm{A}}_j  \widetilde{\bm{X}}}
  =
  b_j,
\end{equation*}
we see that the original problem can be rewritten in the standard SDP form with 
the matrix variable \(\widetilde{\bm{X}} \succeq \bm O\), as
\begin{align*}
  \max_{\widetilde{\bm X}}\quad & \tr{\widetilde{\bm{C}} \widetilde{\bm{X}}}
  \quad \nonumber\\
  \text{s.t.}
  \quad
   &\tr{\widetilde{\bm{A}}_j \widetilde{\bm{X}} } =  b_j,
  \quad \paren{j=1,\dots,m},\\ 
  &\widetilde{\bm X}\succeq \bm O.
\end{align*}

\begin{table}[t]
\centering
\caption{Average of the optimal values across the solved instances for the different layer sizes for Problem A and B.}
\begin{tabular}{rrr}
\toprule
\textbf{$L$} & \textbf{Problem A} & \textbf{Problem B} \\
\midrule
2 & 7.17$\pm$7.24E-05 & 4.93$\pm$1.13E-05 \\
4 & 5.53$\pm$5.31E-06 & 3.86$\pm$1.26E-06 \\
6 & 1.55$\pm$1.57E-07 & 1.92$\pm$7.52E-07 \\
8 & 3.75$\pm$3.69E-09 & 1.56$\pm$1.12E-08 \\
10 & 9.93$\pm$8.48E-11 & 8.55$\pm$9.09E-10 \\
12 & 8.88$\pm$6.14E-13 & $-$4.05$\pm$3.53E-10 \\
16 & 6.38$\pm$7.81E-16 & $-$1.07$\pm$0.73E-09 \\
\bottomrule
\end{tabular}
\label{table:eigenvalue_results}
\end{table}

\section{Empirical Analysis of Constraint Impacts on the Interior-Point Vanishing} \label{sec:prelim_analysis}

To identify the constraints causing the interior-point-vanishing problem, we just directly relaxed the MEM problem~\eqref{eq:sdp_eign} by removing each constraint.
We consider the following two problems:
\paragraph{Problem A: Removal of equality constraints related to ReLU}
\begin{equation*} \label{eq:sdp_2}
\begin{alignedat}{3}
\min_{\bm{P}} \ & \quad\bm{c}^{\top} \bm{P}\sqbra{\bm{x}_L} + c_0 \\
\text{s.t.} \ &\quad \bm{P}\sqbra{\bm{x}_{i+1}} \geq \bm{0} &\quad\paren{i \in\sqbra{L}}, \\ 
&\quad \bm{P}\sqbra{\bm{x}_{i+1}} \geq \bm{W}_i \bm{P}\sqbra{\bm{x}_i} + \bm{b}_i &\quad  \paren{i \in\sqbra{L}}, \\
&\quad \diag{\bm{P}\sqbra{\bm{x}_i \bm{x}_i^{\top}}} - \paren{\bm{l}_i + \bm{u}_i} \odot \bm{P}\sqbra{\bm{x}_i} 
 + \bm{l}_i \odot \bm{u}_i \leq \bm{0}&\quad  \paren{i \in\sqbra{L}},
\\
&\quad \bm{P}\sqbra{1} = 1,~\bm{P} \succeq \bm{O}.
\end{alignedat}
\end{equation*}
\paragraph{Problem B: Applying upper and lower bound constraints only to the input layer} 
\begin{equation*} \label{eq:sdp_3}
\begin{alignedat}{3}
\min_{\bm{P}} &\quad \bm{c}^{\top} \bm{P}\sqbra{\bm{x}_L} + c_0 \\
\text {s.t.} &\quad  \bm{P}\sqbra{\bm{x}_{i+1}} \geq \bm{0} &\quad \paren{i \in\sqbra{L}}, \\
&\quad \bm{P}\sqbra{\bm{x}_{i+1}} \geq \bm{W}_i \bm{P}\sqbra{\bm{x}_i} + \bm{b}_i&\quad \paren{i \in\sqbra{L}}, \\
&\quad \diag{\bm{P}\sqbra{\bm{x}_{i+1} \bm{x}_{i+1}^{\top}} - \bm{W}_i \bm{P}\sqbra{\bm{x}_i \bm{x}_{i+1}^{\top}}} - \bm{b}_i \odot \bm{P}\sqbra{\bm{x}_{i+1}} = \bm{0}&\quad \paren{i \in\sqbra{L}}, \\
&\quad  \diag{\bm{P}\sqbra{\bm{x}_0 \bm{x}_0^{\top}}} - \paren{\bm{l}_0 + \bm{u}_0} \odot \bm{P}\sqbra{\bm{x}_0} + \bm{l}_0 \odot \bm{u}_0 \leq \bm{0},\\
&\quad  \bm{P}\sqbra{1} = 1,~\bm{P} \succeq \bm{O}.
\end{alignedat}
\end{equation*}

\begin{proposition} \label{prop:pro_a}
The feasible region of Problem A is bounded.
\end{proposition}

\begin{proof}
As in the proof of \Cref{lem:trace}, for each $i\in \sqbra{L}$, we have
\begin{equation*}
  \tr{\bm P\sqbra{\bm x_i \bm x_i^\top}} \le \norm{\bm l_i+\bm u_i}_2\sqrt{\tr{P\sqbra{\bm x_i \bm x_i^\top}}} - \bm l_i^\top \bm u_i, 
\end{equation*}
which implies
\begin{equation*}
  \tr{\bm P\sqbra{\bm x_i \bm x_i^\top}} \le \frac{
    \norm{\bm u_i}^2+\norm{\bm l_i}^2
  }{2}
  \quad (i\in \sqbra{L}).
\end{equation*}
Thus, for any feasible solution $\bm P\succeq \bm O$ to Problem A, we have
\begin{align*}
  \lambda_{\max}\paren{\bm P} &\le \tr{\bm P}  \le \frac{1}{2}\sum_{i\in \sqbra{L}} \paren{\norm{ \bm u_i}^2+\norm{\bm l_i}^2}
                              , 
\end{align*}
which ensures that the feasible region of Problem A is bounded.
\end{proof}

\begin{proposition} \label{prop:pro_b}
The feasible region of Problem B is bounded.
\end{proposition}

\begin{proof}
Since the proof of \Cref{lem:trace} only relies on $\bm l_0$ and $\bm u_0$, and
 $\bm l_i$ and $\bm u_i$ for $i\ge 1$ are not used, \Cref{eq:trace_bound} also holds for feasible solutions to Problem B.
Thus, we have that any feasible solution $\bm P\succeq \bm O$ to Problem B satisfies 
\begin{align*}
  \lambda_{\max}\paren{\bm P} &\le \tr{\bm P}  \le \sum_{i\in \sqbra{L}} T_i,
\end{align*}
which completes the proof.
\end{proof}

We check the strict feasibility of the instances of Problems A and B 
by solving the associated problems~\eqref{eq:sdp_eign}. 
We use the same experimental setup as in~\Cref{sec:empirical_analysis}. 
\Cref{table:eigenvalue_results} lists the average of the optimal values across the solved instances (\emph{Avg. Obj.}). As shown, for Problem A, the optimal values were always small regardless of the number of layers. 
For Problem B, while the average of their optimal values tended to decrease as the number of layers increased, their values were larger than those of Problem A, especially for $L=8, 10$.
The results suggest that when a primal-dual gap occurs, relaxation of the upper bound $\bm u_i$ or relaxation of the ReLU equality constraint~\eqref{eq:sdp_relax_relu} may be effective.

\section{Validation on Networks from Prior Works}
\label{sec:appendix_large_network}

We conduct experiments to demonstrate that the interior-point vanishing problem also occurs in DNN models commonly used in prior neural network verification research.

\noindent\textbf{Experimental Setup:}
We evaluated four networks from prior works~\cite{salman2019convex, chiu2023tight}, ranging from 2-layer to 9-layer architectures with 100 neurons per layer. These networks include models trained with different robustification techniques: LPD-MNIST trained using dual formulation training, NOR-MNIST with standard training without robustification, ADV-MNIST as an adversarially trained network, and mnist\_MLP\_9\_100\_ADV as a deep adversarially trained network with 9 layers.  We applied the same experimental methodology as described in~\Cref{sec:empirical_analysis}, solving the strict feasibility verification problem~\eqref{eq:sdp_relax} for each model. Due to computational constraints with larger networks, we employed the standard SDPA solver rather than SDPA-GMP (multi-precision arithmetic). For each network, we solved a single verification instance to assess the presence of interior-point vanishing.

\noindent\textbf{Results:}
The results in~\Cref{tab:prior_work_results} demonstrate that interior-point vanishing occurs across networks used in prior verification research. All tested networks exhibit minimum eigenvalues at or near zero, indicating loss of strict feasibility regardless of the training methodology employed. This universal occurrence suggests that the phenomenon affects networks trained with different robustification techniques, including dual formulation and adversarial training, as well as standard training, indicating it is an inherent property of the SDP relaxation rather than a training artifact.
The interior-point vanishing problem becomes more severe with increased network depth, as evidenced by the 9-layer network showing a more negative minimum eigenvalue ($-1.80 \times 10^{-4}$) compared to the 2-layer networks. This observation confirms our theoretical analysis that deeper networks exacerbate interior-point vanishing.

\begin{table}[h]
\centering
\caption{Interior-point vanishing in networks from prior verification works}
\label{tab:prior_work_results}
\begin{tabular}{lcccc}
\toprule
\textbf{Model} & \textbf{\# Layers} & \textbf{\# Neurons/Layer} & \textbf{Training Method} & \textbf{Objective} \\
\midrule
LPD-MNIST~\cite{salman2019convex, chiu2023tight} & 2 & 100 & Dual formulation & $1.14 \times 10^{-4}$ \\
NOR-MNIST~\cite{salman2019convex, chiu2023tight} & 2 & 100 & - & $5.29 \times 10^{-5}$ \\
ADV-MNIST~\cite{salman2019convex, chiu2023tight} & 2 & 100 & Adversarial Training & $-1.26 \times 10^{-5}$ \\
mnist\_MLP\_9\_100\_ADV~\cite{salman2019convex} & 9 & 100 & Adversarial Training & $-1.80 \times 10^{-4}$ \\
\bottomrule
\end{tabular}
\end{table}

\end{document}